\documentclass{article}

\usepackage{microtype}
\usepackage{graphicx}
\usepackage{subcaption}
\usepackage{booktabs} 

\usepackage{hyperref}

\usepackage[preprint]{icml2026}

\usepackage{amsmath}
\usepackage{amssymb}
\usepackage{mathtools}
\usepackage{amsthm}
\usepackage{multirow}
\usepackage{makecell}
\usepackage{subcaption}
\usepackage{amssymb}
\usepackage{tabularx}
\usepackage{array}
\usepackage[table]{xcolor}
\definecolor{lightgray}{gray}{0.9}

\usepackage[capitalize,noabbrev]{cleveref}

\theoremstyle{plain}
\newtheorem{theorem}{Theorem}[section]

\theoremstyle{definition}

\theoremstyle{remark}

\usepackage[textsize=tiny]{todonotes}

\begin{document}

\twocolumn[
  \icmltitle{VMF-GOS: Geometry-guided virtual Outlier Synthesis \\for Long-Tailed OOD Detection}



  \icmlsetsymbol{equal}{*}

  \begin{icmlauthorlist}
    \icmlauthor{Ningkang Peng}{sch}
    \icmlauthor{Qianfeng Yu}{sch}
    \icmlauthor{Yuhao Zhang}{sch}
    \icmlauthor{Yafei Liu}{sch}
    \icmlauthor{Xiaoqian Peng}{sch2}
    \icmlauthor{Peirong Ma}{sch}
    \icmlauthor{Yi Chen}{sch}
    \icmlauthor{Peiheng Li}{sch}
    \icmlauthor{Yanhui Gu}{sch}
  \end{icmlauthorlist}

  \icmlaffiliation{sch}{School of Computer and Electronic Information, Nanjing Normal University, China}
  \icmlaffiliation{sch2}{School of Artificial Intelligence and Information Technology, Nanjing University of Chinese Medicine, China}

  \icmlcorrespondingauthor{Peirong Ma}{prma@njnu.edu.cn}
  \icmlcorrespondingauthor{Yi Chen}{cs\_chenyi@njnu.edu.cn}
  \icmlcorrespondingauthor{Peiheng Li}{leees@nnu.edu.cn}
  \icmlcorrespondingauthor{Yanhui Gu}{gu@njnu.edu.cn}

  \icmlkeywords{Out of Distribution}

  \vskip 0.3in
]



\printAffiliationsAndNotice{}  

\begin{abstract}
Out-of-Distribution (OOD) detection under long-tailed distributions is a highly challenging task because the scarcity of samples in tail classes leads to blurred decision boundaries in the feature space. Current state-of-the-art (sota) methods typically employ Outlier Exposure (OE) strategies, relying on large-scale real external datasets (such as 80 Million Tiny Images) to regularize the feature space. However, this dependence on external data often becomes infeasible in practical deployment due to high data acquisition costs and privacy sensitivity. To this end, we propose a novel data-free framework aimed at completely eliminating reliance on external datasets while maintaining superior detection performance. We introduce a Geometry-guided virtual Outlier Synthesis (GOS) strategy that models statistical properties using the von Mises-Fisher (vMF) distribution on a hypersphere. Specifically, we locate a low-likelihood annulus in the feature space and perform directional sampling of virtual outliers in this region. Simultaneously, we introduce a new Dual-Granularity Semantic Loss (DGS) that utilizes contrastive learning to maximize the distinction between in-distribution (ID) features and these synthesized boundary outliers. Extensive experiments on benchmarks such as CIFAR-LT demonstrate that our method outperforms sota approaches that utilize external real images.



\end{abstract}
\section{Introduction}
The deployment of deep neural networks (DNNs) in real-world scenarios is increasingly challenged by the simultaneous presence of Long-Tailed Recognition (LTR) and Out-of-Distribution (OOD) \cite{peng,ijcai}. The former arises from extreme class imbalance, which biases models toward head classes while neglecting tail classes; the latter stems from unknown samples in open-world environments, often leading to overconfident but erroneous predictions. In safety-critical applications such as autonomous driving, collaboratively addressing these dual challenges is of paramount importance.

A fundamental bottleneck in long-tailed OOD detection lies in the ambiguity of decision boundaries within the feature space. Due to the scarcity of tail-class samples, their representative features are frequently encroached upon by head classes, leading to severe confusion with OOD samples. Current sota methods (e.g., PASCL\cite{wang2022pascl}, PATT\cite{he2025patt}) typically incorporate Outlier Exposure (OE) strategies, which rely on large-scale surrogate OOD datasets like 80 Million Tiny Images to regularize the feature space. However, the high costs of data acquisition and the sensitivity of data privacy significantly hinder the practical deployment of these methods in restricted or specialized environments.

\begin{figure}[t!]
  \centering 
  \begin{subfigure}[b]{0.49\columnwidth}
    \includegraphics[width=\linewidth]{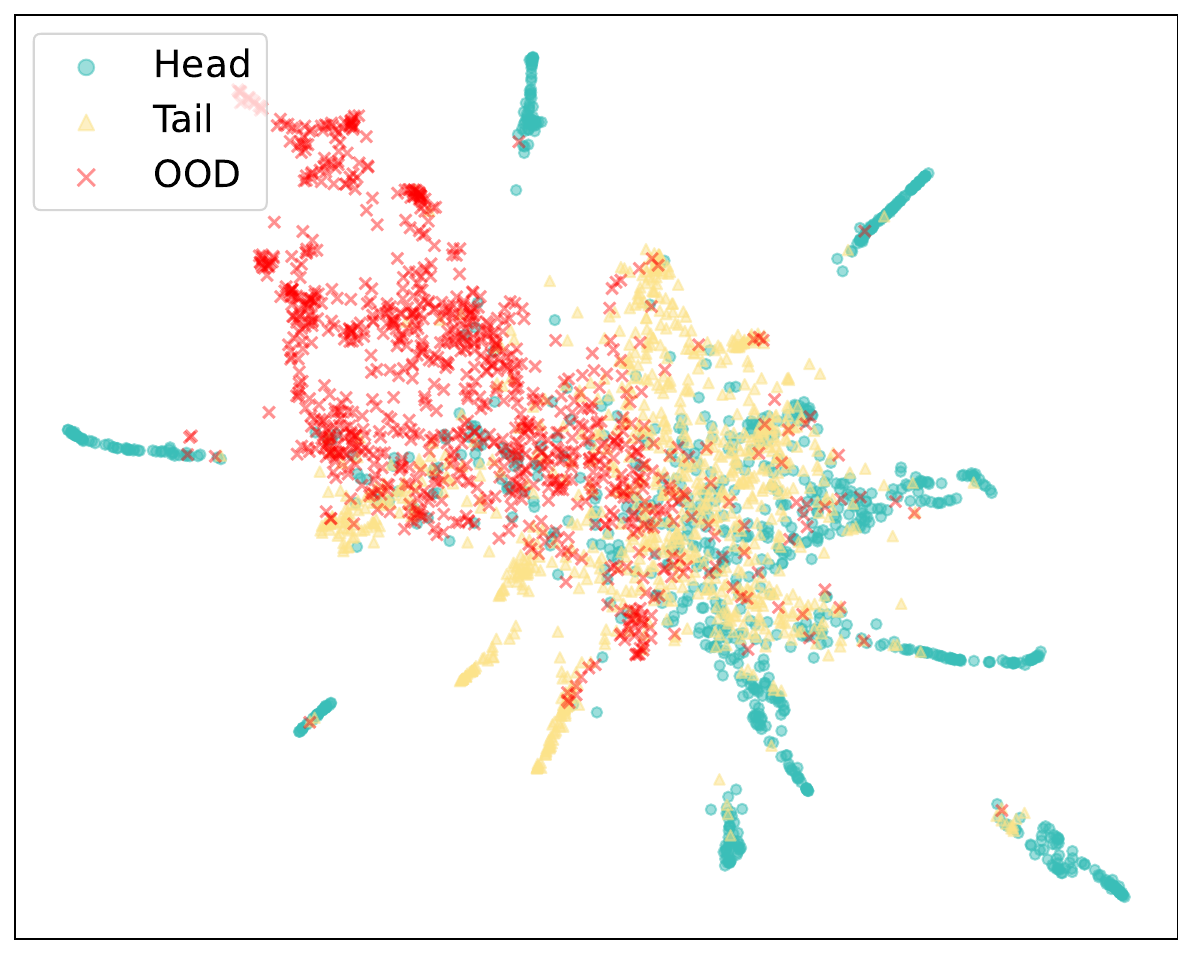}
    \caption{DARL}
    \label{fig:umap_darl}
  \end{subfigure}%
  \hfill 
  \begin{subfigure}[b]{0.49\columnwidth} 
    \includegraphics[width=\linewidth]{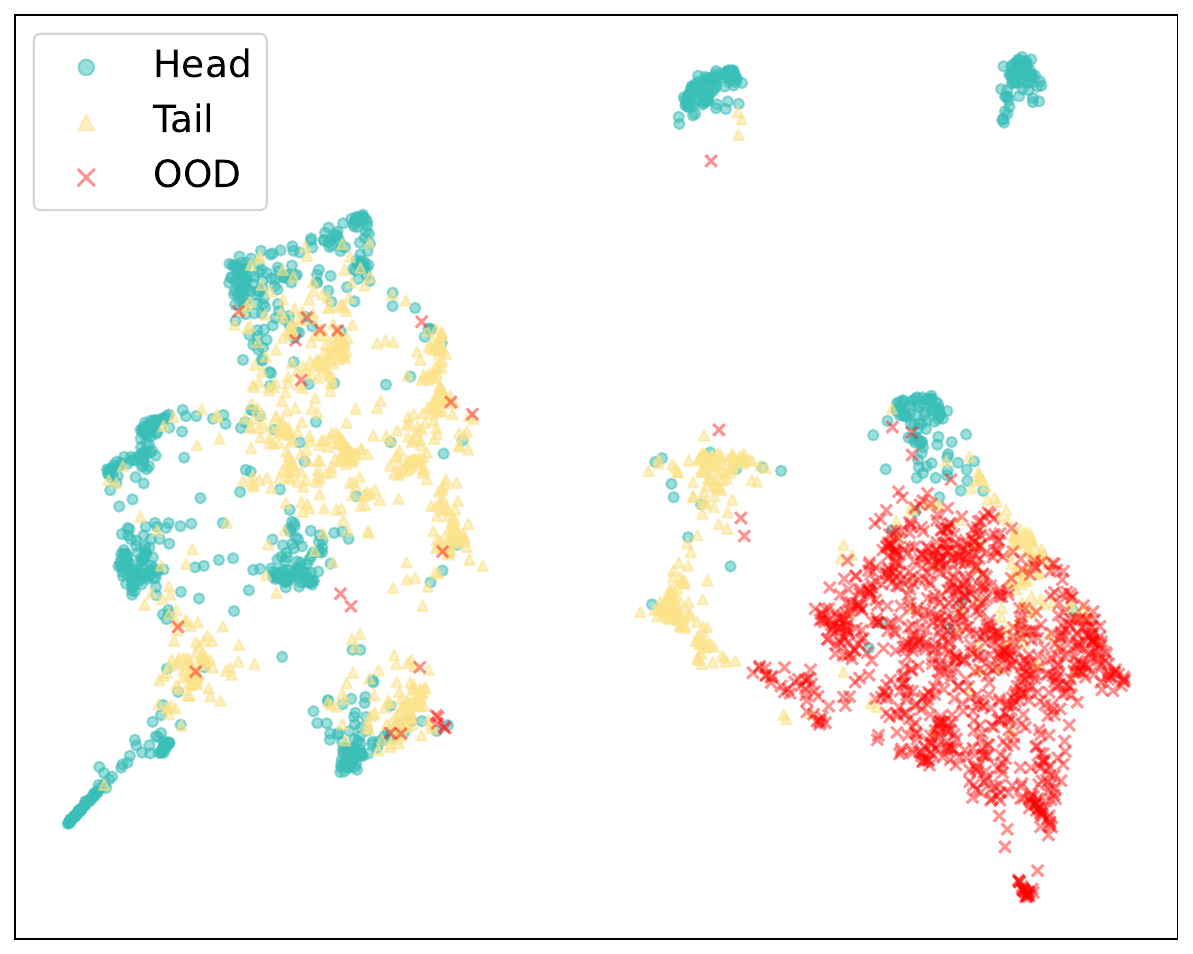}
    \caption{GOS}
    \label{fig:umap_gos}
  \end{subfigure}
  \caption{\textbf{UMAP visualization of feature manifolds on CIFAR-100 with LSUN as the OOD dataset}. (a) The DARL baseline exhibits substantial overlap between tail classes and OOD samples, leading to blurred decision boundaries. (b) Our GOS method fosters structural compactness within ID clusters and establishes a distinct separation margin for OOD detection.}
  \label{fig:feature_manifold_viz}
\end{figure}

While existing methods like DARL\cite{zhang2025darl} attempt to generate pseudo-OOD samples through data mixing (such as Mixup\cite{yun2019cutmix} and CutMix\cite{yun2019cutmix}) to eliminate reliance on external data, this augmentation-based synthesis strategy exhibits significant deficiencies when handling long-tailed distributions. As illustrated in Figure~\ref{fig:umap_darl}, because the pseudo-OOD samples generated by Mixup and CutMix possess strong ID-like attributes rather than being truly distributed in the low-likelihood regions of the ID data, they inevitably suffer from substantial overlap with the manifolds of under-represented tail classes. This results in decision boundaries that are blurred rather than compact, making it difficult for the model to distinguish between high-uncertainty tail classes and genuine OOD samples.

To address this core challenge, we propose the VMF-GOS framework. Our core design philosophy is to eliminate the dependence on external outlier datasets by mining endogenous geometric distribution knowledge to synthesize virtual outliers in low-likelihood regions. Specifically, we introduce two key components for long-tailed OOD detection:
(1)\textit{Geometry-guided virtual Outlier Synthesis(GOS).} Leveraging the asymptotic equivalence between high-dimensional hyperspherical similarity and the Chi-square ($\chi^2$) distribution, we precisely locate the low-likelihood annulus in the feature space for directional virtual outlier sampling. 
(2)\textit{Dual-Granularity Semantic Loss(DGS).} Diverging from prior works like PATT\cite{he2025patt}, which utilize von Mises-Fisher (vMF) distributions primarily for intra-class implicit semantic augmentation to balance ID data, we incorporate synthesized virtual outliers as virtual negative classes into a contrastive learning framework to enforce more compact class-boundary representations.

As illustrated in Figure~\ref{fig:umap_gos}, our proposed VMF-GOS framework not only fosters structural compactness within ID clusters but also establishes a distinct separation margin between tail classes and OOD regions, fundamentally resolving the issue of blurred decision boundaries.

Our primary contributions are summarized as follows:
\begin{itemize}
    \item We propose the GOS strategy. Without requiring any external data, GOS achieves targeted virtual sampling in low-likelihood annular regions of the feature space, thereby enabling a precise characterization of decision boundaries.
    \item We design a DGS loss for the vMF-based synthesis framework. By coupling the implicit semantic expectations of the vMF distribution with explicit boundary constraints of synthesized outliers, DGS effectively enhances the compactness of ID regions under long-tailed distributions.
    \item Extensive experiments on datasets such as CIFAR-LT demonstrate that VMF-GOS significantly outperforms existing methods across key metrics, even surpassing OE strategies that rely on large-scale external real-world images.
\end{itemize}

\section{Related Work}
\textbf{Long-Tailed Recognition.}
LTR aims to mitigate the challenges posed by the extreme class imbalance in training data. Conventional methodologies can be broadly categorized into data-driven and algorithm-driven paradigms. The former encompasses resampling strategies\cite{feng2021ebl,wang2019dcl,shi2023howresampling} and explicit or implicit data augmentation techniques \cite{wang2019isda,chen2022risda,du2024proco}. The latter focuses on alleviating class bias through reweighting \cite{cui2019cbloss,lin2017focalloss} or remargining (logit adjustment) \cite{menon2020logitadjustment,tao2023localandgloballa}. Although these LTR methods achieve remarkable performance in closed-set classification, their efficacy remains limited in real-world open-set scenarios due to the lack of specialized designs for OOD samples.
\newline
\textbf{OOD Detection.}
The objective of OOD detection \cite{nguyen2015ooddetection} is to distinguish whether an input sample belongs to an ID class or an OOD category. Several studies have explored extrapolating OOD feature distributions via adaptive sampling or virtual OOD synthesis \cite{du2022vos,tao2023npos,zhu2023divoe} to construct more precise decision boundaries. Furthermore, post-hoc strategies such as MSP\cite{hendrycks2016msp}, EnergyOE\cite{liu2020energy}, and ODIN\cite{liang2017odin} are frequently integrated to boost performance. However, most existing OOD detectors assume a balanced ID class distribution. Under long-tailed distributions, these models often suffer from bias, leading to the misidentification of rare tail-class ID samples as OOD or the misclassification of OOD samples into frequent head-class ID categories.
\newline
\textbf{Long-Tailed OOD Detection.}
To address the challenges of OOD detection in long-tailed scenarios, PASCL\cite{wang2022pascl} first introduced asymmetric contrastive learning to enlarge the margin between tail classes and OOD data in the feature space. 
EAT\cite{wei2024eat} introduced multiple abstention classes and utilized Cutmix to fuse tail-class data with OOD backgrounds, thereby forcing the model to focus on foreground features. 
COCL\cite{miao2024cocl} further integrated a neighborhood sparsity selection mechanism with calibrated outlier class learning to resolve the confusion between OOD and ID classes at the output layer. 
PATT\cite{he2025patt} models ID semantics through vMF mixture distributions and generates infinite contrastive pairs for semantic enhancement. However, it still requires external OOD samples during training and post-hoc calibration to maintain detection robustness. 
DARL\cite{zhang2025darl} resolves ID-OOD gradient conflicts using conflict-free pseudo-OOD data and ambiguity-aware logit adjustment.
Despite significant progress, how to bolster OOD detection robustness while preventing the erosion of head-class semantic information and maintaining high ID classification accuracy remains an unresolved challenge.




\begin{figure*}[t]
    \centering
    \includegraphics[width=0.99\textwidth]{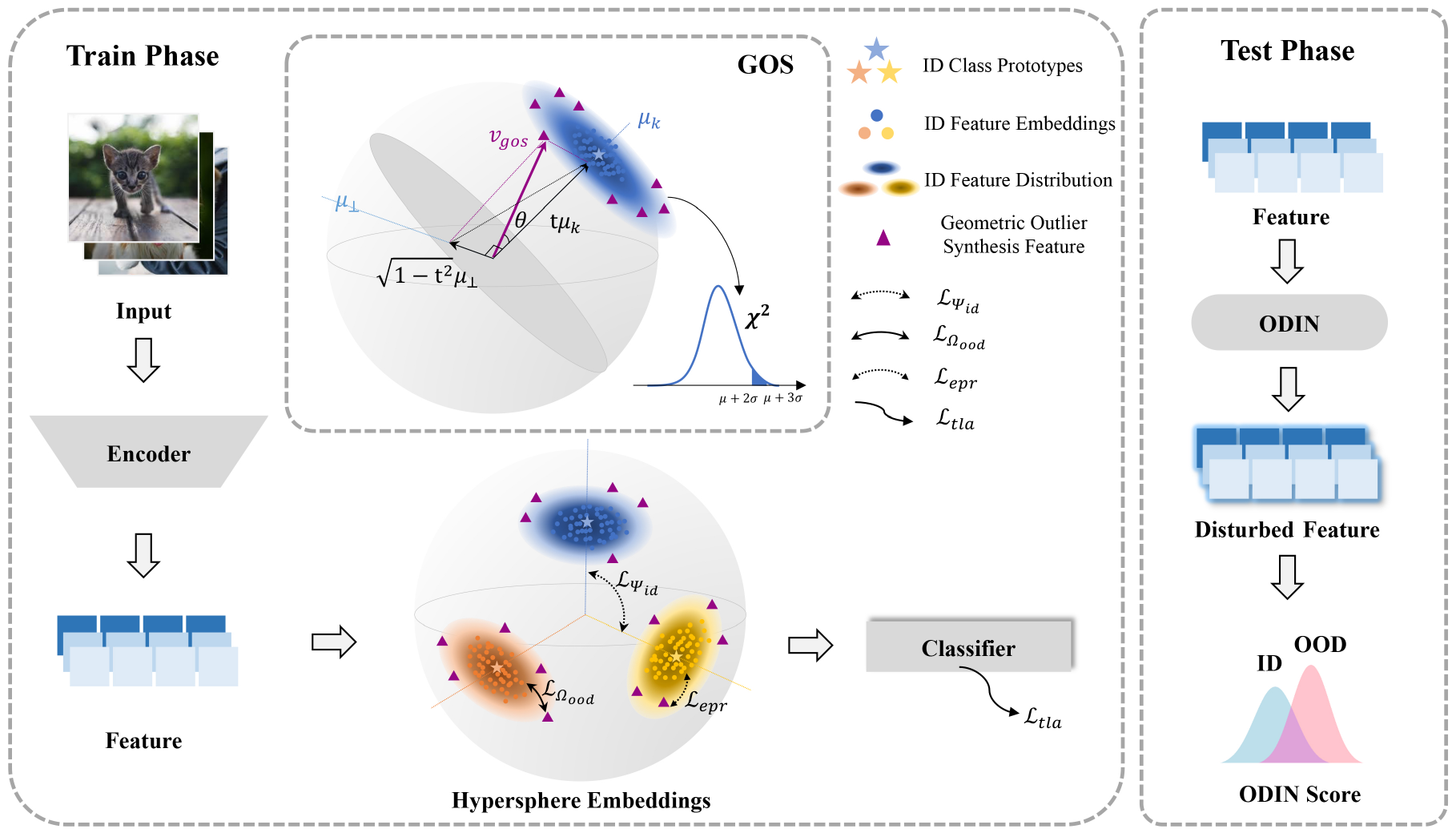}
    \caption{\textbf{Overview of the proposed VMF-GOS framework.}  During the \textbf{Train Phase}, the GOS module leverages the asymptotic equivalence between high-dimensional hyperspherical similarity and the $\chi^{2}$ distribution to directionally sample virtual outliers within low-likelihood annular regions. The model is optimized via a joint objective function consisting of ID semantic alignment ($\mathcal{L}_{\Psi_{id}}$), OOD boundary constraints ($\mathcal{L}_{\Omega_{ood}}$), Energy Polarization Regularization ($\mathcal{L}_{epr}$), and Temperature Scaling-Based Logit Adjustment ($\mathcal{L}_{tla}$), effectively achieving feature decoupling and boundary compression. During the \textbf{Test Phase}, the ODIN post-processing mechanism is employed to amplify the score disparity between ID and OOD samples via gradient-based input perturbations, ensuring robust detection performance under long-tailed distributions.}
    \label{fig:framework}
\end{figure*}
\section{Preliminaries}
\textbf{Task Definition.} Let $\mathcal{D}_{in} = \{(x_i, y_i)\}_{i=1}^N$ denote the training set, characterized by a long-tailed ID set. The input space and label space for ID data are denoted by $\mathcal{X}$ and $\mathcal{Y} = \{1, 2, \dots, K\}$, respectively. In a long-tailed distribution, the number of samples per class $N_y$ exhibits a significant imbalance, typically satisfying $N_1 \gg N_K$. Long-tailed OOD detection aims to jointly learn a feature encoder $f: \mathcal{X} \to \mathcal{Z}$ and a classifier $\varphi: \mathcal{Z} \to \mathcal{Y}$. For any test sample $x \in \mathcal{X}$, the model is required to categorize $x$ into its correct ID class if $x$ is drawn from $\mathcal{D}_{in}$, or identify it as OOD if it originates from an unknown distribution $\mathcal{D}_{out}$. 
\newline
\textbf{Energy-based Score.} Under the Gibbs distribution framework, the energy score $E(z; \tau)$ for a feature vector $z$ is defined as:
\begin{equation}
  E(z; \tau) = -\tau \log \sum_{j=1}^{K} \exp \left( \frac{\varphi_{j}(z)}{\tau} \right),
  \label{eq:energy_score}
\end{equation}
where $\varphi_{j}(z)$ denotes the logit produced by the classifier for the $j$-th class, and $\tau$ is the temperature parameter. This energy-based formulation provides a principled way to map the logit space to a scalar energy value, where lower energy typically indicates a higher likelihood of the sample belonging to ID.

\section{Methodology}
\textbf{Overview of Our Method.} The overview of our proposed VMF-GOS framework is illustrated in Figure~\ref{fig:framework}. Specifically, we first employ a vMF mixture model to explicitly characterize the underlying geometric topology of the ID data. Building upon this, the Geometry-guided virtual Outlier Synthesis (GOS) mechanism exploits the asymptotic equivalence between cosine similarity and the $\chi^{2}$ distribution in high-dimensional hyperspherical spaces. This enables the directional sampling of high-quality virtual negative samples within low-likelihood annular regions, thereby providing the model with precise and rigorous boundary constraints. During the optimization phase, we introduce a tripartite loss objective. DGS reinforces the compactness of ID regions by integrating implicit semantic augmentation with explicit boundary regularization. Temperature Scaling-Based Logit Adjustment (TLA) rectifies the predictive bias inherent in long-tailed distributions and bolsters ID confidence for minority classes. Energy Polarization Regularization (EPR) explicitly polarizes the energy gap between ID samples and virtual outliers, facilitating effective feature decoupling in the energy space. Finally, at the inference stage, we incorporate the ODIN post-processing strategy to further bolster the discriminative robustness for tail categories. Through the synergistic optimization of geometric structures and energy landscapes, VMF-GOS constructs a robust and class-balanced decision boundary for long-tailed OOD detection.

\subsection{Geometry-Guided Virtual Outlier Synthesis}
\textbf{Hyperspherical Distribution Modeling.} Considering that high-dimensional features exhibit prominent directional characteristics after normalization, we employ a vMF mixture model to characterize the geometric structure of ID data. For a $d$-dimensional unit vector $z \in \mathbb{S}^{d-1}$ in the embedding space, the probability density function of the vMF distribution is defined as:
\begin{equation}
  P_d(z \mid \mu_y, \kappa_y)=Z_d(\kappa_y)\exp(\kappa_y \mu_y^\top z),
  \label{eq:pdf_vmf}
\end{equation}
where $\mu_y$ denotes the mean direction vector of class $y$, $\kappa_y \ge 0$ is the concentration parameter that measures the degree of feature clustering around the mean direction, and $Z_d(\kappa_y)$ represents the normalization constant. Based on the aforementioned single-class model, we describe the global feature distribution through a vMF Mixture Model. The total probability density function is represented as the weighted sum of all class-conditional densities:
\begin{equation}
  P_d(z)=\sum_{y=1}^{K}\pi_y P_d(z \mid \mu_y, \kappa_y)=\sum_{y=1}^{K}\pi_y Z_d(\kappa_y)\exp(\kappa_y \mu_y^\top z),
  \label{eq:pdf_vmf_mixture}
\end{equation}
where $\pi_y=N_y/N$ denotes the prior probability of class $y$.
\newline
\textbf{Synthesis Strategy.} Departing from conventional approaches that model feature spaces as parametric Gaussian distributions, we leverage the asymptotic equivalence between cosine similarity and the $\chi^{2}$ distribution in high-dimensional hyperspherical spaces. We propose a virtual feature synthesis strategy within the low-likelihood annulus to precisely characterize decision boundaries in the absence of auxiliary OOD data. On a $d$-dimensional unit hypersphere, when the concentration parameter $\kappa$ is sufficiently large, the geometric displacement of samples from the class center $\mu_{k}$ can be mapped to a $\chi^{2}$ distribution. Specifically, we determine the target displacement within the boundary regions of the $\chi^{2}$ distribution:
$ \xi \sim \text{Uniform}(\mu_{\chi^{2}} + 2\sigma_{\chi^{2}}, \mu_{\chi^{2}} + 3\sigma_{\chi^{2}})$
This is subsequently transformed into a cosine similarity scalar $t$ via an inverse mapping function $t = 1 - \frac{\xi}{2\kappa}$. To generate semantically relevant and geometrically rigorous outliers, we apply Gram-Schmidt orthogonalization to a random noise vector to synthesize a tangential component $v_{\perp}$ that is strictly orthogonal to the class center $\mu_{k}$. The resulting feature is:
\begin{equation}
  z_{gos} = t\mu_{k} + \sqrt{1-t^{2}}v_{\perp}. 
  \label{eq:gos}
\end{equation}
The detailed mathematical derivation is provided in Appendix~\ref{sec:derivation_equ} and ~\ref{sec:derivation_norm}. Notably, we perform balanced synthesis across all classes, ensuring that even tail categories with limited samples are equipped with sufficient virtual negative constraints to prevent boundary over-expansion.

\begin{theorem}[Asymptotic $\chi^2$ Equivalence]
    Consider a $d$-dimensional feature vector $z \in \mathbb{S}^{d-1}$ following a vMF distribution with mean direction $\mu_k$ and concentration parameter $\kappa$. Under high-dimensional concentration ($d \to \infty$ and large $\kappa$), the scaled angular displacement $\xi = 2\kappa(1 - \mu_k^\top z)$ asymptotically follows a Chi-square distribution with $d-1$ degrees of freedom:
    \begin{equation}
        \xi = 2\kappa(1 - t) \sim \chi_{d-1}^2,
    \end{equation}
    where $t = \mu_k^\top z$ denotes the cosine similarity.
\end{theorem}

\begin{proof}[Proof Sketch]
By rotating the coordinate system such that $\mu_k = (1, 0, \dots, 0)^\top$, the similarity satisfies $t = \sqrt{1 - \sum_{i=2}^d z_i^2}$. Under the localization assumption, a second-order Taylor expansion yields $t \approx 1 - \frac{1}{2} \sum_{i=2}^d z_i^2$. Substituting this into the vMF probability density function, the projected components $\sqrt{\kappa}z_i$ ($i=2,\dots,d$) converge to i.i.d. standard normal variables $\mathcal{N}(0, 1)$. Thus, the sum of squares $\xi = 2\kappa(1-t) \approx \sum_{i=2}^d (\sqrt{\kappa}z_i)^2$ asymptotically follows $\chi_{d-1}^2$. See Appendix for the complete derivation.
\end{proof}
Guided by this theorem, the GOS strategy samples $\xi$ from the low-likelihood tail regions of the $\chi_{d-1}^2$ distribution, specifically $\xi \sim \text{Uniform}(\mu_{\chi^2} + 2\sigma_{\chi^2}, \mu_{\chi^2} + 3\sigma_{\chi^2})$, and maps it back to similarity scalars via $t = 1 - \frac{\xi}{2\kappa}$.

\subsection{Optimization Targets}
\textbf{Dual-Granularity Semantic Loss.} To synergistically exploit global distribution information and local synthesized samples, inspired by \cite{he2025patt}, we design the DGS Loss. The detailed mathematical derivation is provided in Appendix~\ref{sec:derivation_dgs}. This loss integrates geometrically synthesized outliers as explicit constraints into the contrastive learning framework to reinforce the compactness of ID boundaries:
\begin{equation}
  \mathcal{L}_{dgs}(z_{i}, y) = -\log \frac{\Psi_{id}(z_{i}, y)}{\sum_{j=1}^K \Psi_{id}(z_{i}, j) + \sum_{m=1}^M \Omega_{ood}(z_{i}, z_{m}^{gos})}.
  \label{eq:dgs_loss}
\end{equation}
Intuitively, the numerator $\Psi_{id}$ encourages features to align with the continuous semantic region of their ground-truth class, while the denominator's $\Omega_{ood}$ term imposes a hard-negative boundary constraint using synthesized outliers.

The implicit ID term $\Psi_{id}$ leverages the closed-form expectation of the vMF distribution to measure the normalized similarity between $z_{i}$ and the continuous semantic region of class $j$:
\begin{equation}
  \Psi_{id}(z_{i}, j) = \frac{\pi_{j} Z_{d}(\tilde{\kappa}_{y}) Z_{d}(\kappa_{j})}{\pi_{y} Z_{d}(\kappa_{j}) Z_{d}(\tilde{\kappa}_{j})},
  \label{eq:psi_id}
\end{equation}
where $Z_{d}(\cdot)$ is the normalization factor of the $d$-dimensional vMF distribution, and $\tilde{\kappa}_j = \|\kappa_j\mu_j + z_i / \tau\|_2$ denotes the dynamic concentration parameter updated after observing the sample. 

The explicit OOD term $\Omega_{ood}$ characterizes the relative likelihood intensity of $z_{i}$ subject to the perturbation from the synthesized outlier $z_{m}^{gos}$:
\begin{equation}
  \Omega_{ood}(z_{i}, z_{m}^{gos}) = \frac{Z_{d}(\tilde{\kappa}_{y}) Z_{d}(\kappa_{k(m)})}{\pi_{y} Z_{d}(\kappa_{k(m)}) Z_{d}(\tilde{\kappa}_{m}^{gos})},
  \label{eq:omega_ood}
\end{equation}
where $k(m)$ denotes the index of the source anchor class from which the $m$-th virtual outlier is derived. The term $\tilde{\kappa}_{m}^{gos} = \|(z_m^{gos} + z_i) / \tau\|_2$ encapsulates the interaction intensity between the query feature $z_i$ and the synthesized outlier $z_m^{gos}$ in the embedding space. The class prior $\pi_{y_i}$ in the denominator serves as a pivotal penalty weight. For tail classes with small $\pi_{y_i}$, the loss assigned to OOD-proximal samples is amplified, effectively suppressing the boundary expansion that typically plagues long-tailed distributions.
\newline
\textbf{Temperature Scaling-Based Logit Adjustment.} To alleviate the predictive bias caused by long-tailed distributions, we follow the recent work of \cite{he2025patt} and incorporate a temperature scaling factor $\epsilon$ to optimize class margins and enhance ID confidence. Given the class prior $\pi_{y}=N_{y}/N$, the TLA loss function is defined as follows:
\begin{equation}
  \mathcal{L}_{tla}(z,y)=-\log\frac{\pi_{y}\exp(\varphi_{y}(z)/\epsilon)}{\sum_{j=1}^{K}\pi_{j}\exp(\varphi_{j}(z)/\epsilon)},
  \label{eq:tla_loss}
\end{equation}
where $\epsilon$ denotes the temperature coefficient for smoothing the probability distribution. By explicitly compensating for prior bias, TLA ensures that the model generates high-confidence outputs even for tail classes with limited samples.
\newline
\textbf{Energy Polarization Regularization.} To effectively distinguish ID data from the GOS features in the energy space, we follow prior work of \cite{du2022vos} and propose the Energy Polarization Regularization (EPR). The core motivation of EPR is to shape a discriminative energy landscape by imposing explicit contrastive constraints in the energy space. Specifically, we introduce a two-layer non-linear Multi-Layer Perceptron (MLP) to define a mapping function $\phi(\cdot)$. The EPR loss is formulated as follows:
\begin{equation}
  \begin{aligned}
    \mathcal{L}_{epr} &= \mathbb{E}_{z \sim \mathcal{D}_{gos}} [-\log \sigma(\phi(E(z; \epsilon)))] \\
    &+ \mathbb{E}_{z \sim \mathcal{D}_{in}} [-\log (1 - \sigma(\phi(E(z; \epsilon))))],
  \end{aligned}
  \label{eq:esl_loss}
\end{equation}
where $\sigma(\cdot)$ denotes the sigmoid function and $\mathcal{D}_{gos}$ denotes the set of virtual outliers synthesized through our GOS strategy. 
\newline
\textbf{Joint Optimization.} Finally, the overall loss function of our framework is as follows:
\begin{equation}
  \mathcal{L}=\mathcal{L}_{dgs}+\alpha\mathcal{L}_{tla}+\beta\mathcal{L}_{epr},
  \label{eq:total_loss}
\end{equation}
where $\alpha$ and $\beta$ are hyperparameters that balance the contributions of each  components. 

\subsection{Post-Hoc OOD Detection}
To further enhance the discriminative margin between ID and OOD samples in the feature space, we incorporate adversarial perturbations as proposed in the ODIN\cite{liang2017odin} framework. For each test sample $x$, a perturbed sample $\hat{x}$ is generated by computing the gradient of the log-softmax score with respect to the input for the predicted class $\hat{y}$:
\begin{equation}
  \hat{x}=x-\eta\cdot\text{sign}(-\nabla_{x}\log p(\hat{y}\mid x;\tau_{o})),
  \label{eq:odin_sample}
\end{equation}
where $\eta$ denotes the perturbation magnitude and $\tau_{o}$ is the temperature scaling parameter for ODIN. The perturbed sample $\hat{x}$ is then re-fed into the model to obtain its feature representation $\hat{z}$, based on which the final OOD score $S(\hat{x})$ is derived. Grounded in energy-based model theory, the scoring function is defined using the negative energy function:
\begin{equation}
  S(\hat{x};\tau_{o})= - \tau_{o}\log\sum_{j=1}^{K}\exp(\varphi_{j}(\hat{z})/\tau_{o}).
  \label{eq:odin_energy}
\end{equation}
A higher value of $S(\hat{x};\tau_{o})$ indicates that the sample is more likely to be classified as ID, whereas a lower score signifies an OOD sample.

\section{Experiments}

\subsection{Experiment Settings}
\textbf{Datasets.} We evaluate our method on three widely used long-tailed benchmark datasets: CIFAR10-LT and CIFAR100-LT\cite{cao2019cifar-lt}. For CIFAR10/100-LT, we follow the exponential imbalance decay protocol with a default imbalance ratio of 100. The standard test sets of CIFAR10 and CIFAR100 are utilized as the ID test data $D_{in}^{test}$.For OOD detection evaluation, we follow the SC-OOD benchmark\cite{yang2021sc-ood} and adopt five diverse datasets as $D_{out}^{test}$ for CIFAR-LT: Textures\cite{cimpoi2014textures}, SVHN\cite{netzer2011svhn}, Tiny ImageNet\cite{le2015tiny}, LSUN\cite{yu2015lsun}, and Places365\cite{zhou2017places365}. Furthermore, we conduct near-OOD experiments\cite{yang2021sc-ood}by using CIFAR-100 as $D_{out}^{test}$ for CIFAR10-LT and vice versa. 
\newline
\textbf{Evaluation Protocols.} Following \cite{wang2022pascl} and \cite{wei2024eat}, we use the below evaluation measures: 
\begin{itemize}
    \item \textbf{AUROC} ($\uparrow$): Area Under the Receiver Operating Characteristic curve. It represents the probability that a randomly selected OOD sample is assigned a higher detection score than a randomly selected ID sample.
    \item \textbf{AUPR} ($\uparrow$): Area Under the Precision-Recall curve. This metric summarizes the precision-recall trade-off, reflecting the average precision across various recall levels.
    \item \textbf{FPR@TPR$n$} ($\downarrow$): The False Positive Rate (FPR) when the True Positive Rate (TPR) reaches $n$. Abbreviated as \textbf{FPR95}, \textbf{FPR@TPR}$95$ quantifies the proportion of ID samples misclassified as OOD while ensuring that $95\%$ of OOD samples are correctly identified.
    \item \textbf{ACC@TPR$n$} ($\uparrow$): The classification accuracy on the remaining ID samples after successfully detecting $n$ of the OOD samples. For $n=95\%$, it is denoted as \textbf{ACC95}.
    \item \textbf{ACC@FPR$n$} ($\uparrow$): The classification accuracy on the remaining ID samples when a false alarm rate of $n$ is allowed for ID samples. Notably, \textbf{ACC@FPR}$0$ is equivalent to the standard accuracy (\textbf{ACC}), representing the model's original performance on the full test set.
\end{itemize}

\textbf{Implementation Details.}
Following the settings in \cite{wang2022pascl}, we utilize ResNet-18 for CIFAR-LT. What's more, we use Adam optimizer with an initial learning rate $1 \times 10^{-3}$ for CIFAR-based experiments. We train each model for $100$ epochs using a batch size of $128$.

\begin{table}[t!]
    \caption{Results on CIFAR10-LT using ResNet18. The best results are shown in bold. Mean and standard deviation over six random runs are reported for OE, PASCL, EAT and our method. Average means the results averaged across six different $\mathcal{D}_{\text{out}}^{\text{test}}$ sets.}
    \begin{subtable}{\columnwidth}
        \centering
        \caption{OOD detection results and in-distribution classification results in terms of ACC95.}
        \resizebox{\columnwidth}{!}
        {
            \begin{tabular}{c c >{$}c<{$} >{$}c<{$} >{$}c<{$} >{$}c<{$}}
                \toprule
                $\mathcal{D}_{\text{out}}^{\text{test}}$ & \textbf{Method} & \textbf{AUROC($\uparrow$)} & \textbf{AUPR($\uparrow$)} & \textbf{FPR95($\downarrow$)} & \textbf{ACC95($\uparrow$)} \\ 
                \midrule
                \multirow{4}{*}{Texture} 
                    & OE      & 92.59_{\pm0.42} & 83.32_{\pm1.67} & 25.10_{\pm1.08} & 84.52_{\pm0.76} \\
                    & PASCL  & 93.16_{\pm0.37} & 84.80_{\pm1.50} & 23.26_{\pm0.91} & 85.86_{\pm0.72} \\
                    & EAT    & 95.44_{\pm0.46} & 92.28_{\pm0.95} & 21.50_{\pm1.50} & 87.00_{\pm0.66} \\
                    \rowcolor{lightgray} \cellcolor{white}
                    & \textbf{Ours} & \mathbf{90.85}_{\pm0.40} & \mathbf{84.75}_{\pm0.73} & \mathbf{36.08}_{\pm1.95} & \mathbf{94.96}_{\pm0.65} \\
                \midrule
                \multirow{4}{*}{SVHN}           
                    & OE      & 95.10_{\pm1.01} & 97.14_{\pm0.81} & 16.15_{\pm1.52} & 81.33_{\pm0.81} \\
                    & PASCL  & 96.63_{\pm0.90} & 98.06_{\pm0.56} & 12.18_{\pm3.33} & 82.72_{\pm1.51} \\
                    & EAT    & 97.92_{\pm0.36} & 99.06_{\pm0.20} & 9.87_{\pm2.06} & 84.39_{\pm0.51} \\
                    \rowcolor{lightgray} \cellcolor{white}
                    & \textbf{Ours} & \mathbf{98.98}_{\pm0.27} & \mathbf{99.60}_{\pm0.08} & \mathbf{5.42}_{\pm1.43} & \mathbf{86.92}_{\pm0.58} \\
                \midrule
                \multirow{4}{*}{CIFAR100}       
                    & OE      & 83.40_{\pm0.30} & 80.93_{\pm0.57} & 56.96_{\pm0.91} & 94.56_{\pm0.57} \\
                    & PASCL  & 84.43_{\pm0.23} & 82.99_{\pm0.48} & 57.27_{\pm0.88} & 94.48_{\pm0.31} \\
                    & EAT    & 85.93_{\pm0.15} & 86.10_{\pm0.35} & 54.13_{\pm0.63} & 95.81_{\pm0.41} \\
                    \rowcolor{lightgray} \cellcolor{white}
                    & \textbf{Ours} & \mathbf{84.99}_{\pm0.26} & \mathbf{82.91}_{\pm0.19} & \mathbf{51.18}_{\pm0.78} & \mathbf{96.95}_{\pm0.22} \\
                \midrule
                \multirow{4}{*}{\makecell{Tiny \\ ImageNet}} 
                    & OE      & 86.14_{\pm0.29} & 79.33_{\pm0.65} & 47.78_{\pm0.72} & 91.19_{\pm0.33} \\
                    & PASCL  & 87.14_{\pm0.18} & 81.54_{\pm0.38} & 47.69_{\pm0.59} & 91.20_{\pm0.35}\\
                    & EAT    & 89.11_{\pm0.34} & 85.43_{\pm0.58} & 41.75_{\pm0.68} & 91.67_{\pm0.65} \\
                    \rowcolor{lightgray} \cellcolor{white}
                    & \textbf{Ours} & \mathbf{87.06}_{\pm0.30} & \mathbf{85.50}_{\pm0.13} & \mathbf{46.49}_{\pm1.15} & \mathbf{96.28}_{\pm0.34} \\
                \midrule
                \multirow{4}{*}{LSUN}           
                    & OE      & 91.35_{\pm0.23} & 87.62_{\pm0.82} & 27.86_{\pm0.68} & 85.49_{\pm0.69} \\
                    & PASCL  & 93.17_{\pm0.15} & 91.76_{\pm0.53} & 26.40_{\pm1.00} & 86.67_{\pm0.90} \\
                    & EAT    & 95.13_{\pm0.43} & 94.12_{\pm0.61} & 19.72_{\pm1.61} & 86.68_{\pm0.64} \\
                    \rowcolor{lightgray} \cellcolor{white}
                    & \textbf{Ours} & \mathbf{97.18}_{\pm0.43} & \mathbf{96.87}_{\pm0.19} & \mathbf{12.56}_{\pm1.74} & \mathbf{89.30}_{\pm0.77} \\
                \midrule
                \multirow{4}{*}{Places365}      
                    & OE      & 90.07_{\pm0.26} & 95.15_{\pm0.24} & 34.04_{\pm0.91} & 87.07_{\pm0.53} \\
                    & PASCL  & 91.43_{\pm0.17} & 96.28_{\pm0.14} & 33.40_{\pm0.88} & 87.87_{\pm0.71} \\
                    & EAT    & 93.68_{\pm0.27} & 97.42_{\pm0.14} & 26.03_{\pm0.92} & 87.64_{\pm0.68} \\
                    \rowcolor{lightgray} \cellcolor{white}
                    & \textbf{Ours} & \mathbf{90.86}_{\pm0.11} & \mathbf{89.23}_{\pm0.27} & \mathbf{36.35}_{\pm0.89} & \mathbf{94.95}_{\pm0.57} \\
                \midrule
                \multirow{4}{*}{Average}        
                    & OE      & 89.77_{\pm0.27} & 87.25_{\pm0.61} & 34.65_{\pm0.46} & 87.36_{\pm0.51} \\
                    & PASCL  & 90.99_{\pm0.19} & 89.24_{\pm0.34} & 33.36_{\pm0.79} & 88.13_{\pm0.56} \\
                    & EAT    & 92.87_{\pm0.33} & 92.40_{\pm0.47} & 28.83_{\pm1.23} & 88.86_{\pm0.59} \\
                    \rowcolor{lightgray} \cellcolor{white}
                    & \textbf{Ours} & \mathbf{91.65}_{\pm0.39} & \mathbf{89.81}_{\pm0.26} & \mathbf{31.35}_{\pm0.68} & \mathbf{93.23}_{\pm0.45} \\
                \bottomrule
            \end{tabular}
        }
        \label{tab:cifar10a}
    \end{subtable}
    
    \vspace{0.5em}
    
    \begin{subtable}{\columnwidth}
    \centering
    \caption{in-distribution classification results in terms of ACC@FPR$n$.}
    \resizebox{\columnwidth}{!}{
        \begin{tabular}{c >{$}c<{$} >{$}c<{$} >{$}c<{$} >{$}c<{$}}
            \toprule
            \multirow{2}{*}{\textbf{Method}} & \multicolumn{4}{c}{\textbf{ACC@FPR$n$ ($\uparrow$)}} \\ 
            & 0 & 0.001 & 0.01 & 0.1 \\ 
            \midrule
            OE    & 73.54_{\pm0.77} & 73.90_{\pm0.77} & 74.46_{\pm0.81} & 78.88_{\pm0.66} \\
            PASCL & 77.08_{\pm1.01} & 77.13_{\pm1.02} & 77.64_{\pm0.99} & 81.96_{\pm0.85} \\
            EAT   & 81.31_{\pm0.26} & 81.36_{\pm0.25} & 81.81_{\pm0.26} & 84.40_{\pm0.28} \\
            \rowcolor{lightgray}
            \textbf{Ours} & \mathbf{84.80}_{\pm0.34} & \mathbf{84.87}_{\pm0.33} & \mathbf{85.37}_{\pm0.23} & \mathbf{88.51}_{\pm0.28} \\
            \bottomrule
        \end{tabular}
        }
    \label{tab:cifar10b}
    \end{subtable}

    \vspace{0.5em}
    
    \begin{subtable}{\columnwidth}
        \centering
        \caption{Comparison with other methods.}
        \resizebox{\columnwidth}{!}{
            \begin{tabular}{c >{$}c<{$} >{$}c<{$} >{$}c<{$} >{$}c<{$}}
                \toprule
                \textbf{Method} & \textbf{AUROC($\uparrow$)} & \textbf{AUPR($\uparrow$)} & \textbf{FPR95($\downarrow$)} & \textbf{ACC($\uparrow$)} \\
                \midrule
                    ST(MSP)           & 72.28 & 70.27 & 66.07 & 72.34 \\
                    EnergyOE          & 89.31 & 88.92 & 40.88 & 74.68 \\
                    OE                & 89.77_{\pm0.27} & 87.25_{\pm0.61} & 34.65_{\pm0.46} & 73.84_{\pm0.77} \\
                    PASCL             & 90.99_{\pm0.19} & 89.24_{\pm0.34} & 33.36_{\pm0.79} & 77.08_{\pm1.01} \\
                    EAT               & \mathbf{92.87}_{\pm0.33} & \mathbf{92.40}_{\pm0.47} & \mathbf{28.83}_{\pm1.23} & 81.31_{\pm0.26} \\
                    PATT              & 91.62_{\pm0.50} & \underline{90.48}_{\pm0.30} & \underline{29.89}_{\pm1.40} & \underline{84.77}_{\pm0.20} \\
                    DARL              & 88.65_{\pm1.12} & 86.17_{\pm0.89} & 39.49_{\pm1.04} & 77.55_{\pm0.43} \\
                    \rowcolor{lightgray}
                    \textbf{Ours}     & \underline{91.65}_{\pm0.39} & 89.81_{\pm0.26} & 31.35_{\pm0.68} & \mathbf{84.80}_{\pm0.28} \\
                \bottomrule
            \end{tabular}
        }
        \label{tab:cifar10c}
    \end{subtable}
    \label{tab:cifar10_abc}
\end{table}

\subsection{Main Results}
Table~\ref{tab:cifar10_abc} and Table~\ref{tab:cifar100_abc} report the experimental results on the CIFAR10-LT and CIFAR100-LT datasets, respectively. Each table contains three sub-tables: (a) presents the AUROC, AUPR, FPR95, and ACC95 metrics along with their averages across six different $D_{out}^{test}$ sets; (b) compares the ACC@FPRn metrics across various n values; (c) summarizes the comprehensive performance of our method against sota approaches in both OOD detection and ID classification. For a fair comparison, the data for MSP\cite{hendrycks2016msp}, EnergyOE\cite{liu2020energy}, OE\cite{hendrycks2018oe}, PASCL\cite{wang2022pascl}, EAT\cite{wei2024eat}, and PATT\cite{he2025patt} are cited from their original papers. The results for DARL\cite{zhang2025darl} and our method are reported as the mean of 6 experimental runs using a ResNet-18 backbone.

As shown in Table~\ref{tab:cifar10_abc}, our method significantly outperforms existing baselines across the majority of metrics. In Table~\ref{tab:cifar10a}, our approach achieves the best performance in ACC95 (93.23\%) and ACC (84.80\%). Notably, in the ACC@FPR$n$ evaluation in Table~\ref{tab:cifar10b}, our method remains highly competitive across all n values. For instance, at n=0.1, our accuracy reaches 88.51\%, substantially exceeding EAT (84.40\%) and PASCL (81.96\%). While our AUROC in Table~\ref{tab:cifar10c} is slightly lower than that of EAT, our method significantly enhances the recognition accuracy of ID long-tailed classes while maintaining competitive OOD detection performance, demonstrating a superior balance for practical applications.

In the more challenging CIFAR100-LT task, which features higher class diversity and data imbalance (Table~\ref{tab:cifar100_abc}), the advantages of our method are even more pronounced. In the full-method comparison in Table~\ref{tab:cifar100c}, our approach achieves sota performance across all core metrics. Specifically, our AUROC is 0.81\% higher than PATT (77.06\% vs 76.25\%), and the improvement in AUPR reaches 3.09\% (73.46\% vs 70.37\%). Furthermore, FPR95 is reduced to 61.99\%, and ACC is improved to 51.11\%. In Table~\ref{tab:cifar100a} and Table~\ref{tab:cifar100b}, our method achieves leap-forward improvements in ACC95 (77.66\%) and ACC@FPR$n$. For example, our ACC95 is 5.23\% higher than PASCL (77.66\% vs 72.43\%), and at n=0.1, our ACC@FPRn outperforms EAT by 6.30\% (54.69\% vs 48.39\%). These results validate that as task complexity increases, our method exhibits stronger capabilities in capturing complex long-tailed distributions and maintaining sharp OOD boundaries.

\begin{table}[t!]
    \caption{Results on CIFAR100-LT using ResNet18. The best results are shown in bold. Mean and standard deviation over six random runs are reported for OE, PASCL, EAT and our method. Average meansthe results averaged across six different $\mathcal{D}_{\text{out}}^{\text{test}}$ sets.}
    \begin{subtable}{\columnwidth}
        \centering
        \caption{OOD detection results and in-distribution classification results in terms of ACC95.}
        \resizebox{\columnwidth}{!}{
            \begin{tabular}{c c >{$}c<{$} >{$}c<{$} >{$}c<{$} >{$}c<{$}}
                \toprule
                $\mathcal{D}_{\text{out}}^{\text{test}}$ & \textbf{Method} & \textbf{AUROC($\uparrow$)} & \textbf{AUPR($\uparrow$)} & \textbf{FPR95($\downarrow$)} & \textbf{ACC95($\uparrow$)} \\ 
                \midrule
                \multirow{4}{*}{Texture} 
                    & OE    & 76.71_{\pm1.20} & 58.79_{\pm1.39} & 68.28_{\pm1.53} & 71.43_{\pm1.58} \\
                    & PASCL & 76.01_{\pm0.66} & 58.12_{\pm1.06} & 67.43_{\pm1.93} & 73.11_{\pm1.55} \\
                    & EAT   & 80.27_{\pm0.76} & 71.76_{\pm1.56} & 67.53_{\pm0.64} & 73.76_{\pm0.75} \\
                    \rowcolor{lightgray} \cellcolor{white}
                    & \textbf{Ours} & \mathbf{71.10}_{\pm0.98} & \mathbf{59.70}_{\pm0.96} & \mathbf{85.43}_{\pm2.93} & \mathbf{88.36}_{\pm1.29} \\
                \midrule
                \multirow{4}{*}{SVHN}            
                    & OE    & 77.61_{\pm3.26} & 86.82_{\pm2.50} & 58.04_{\pm4.82} & 64.27_{\pm3.26} \\
                    & PASCL & 80.19_{\pm2.19} & 88.49_{\pm1.59} & 53.45_{\pm3.60} & 64.50_{\pm1.87} \\
                    & EAT   & 83.11_{\pm2.83} & 89.71_{\pm2.08} & 47.78_{\pm4.87} & 61.67_{\pm2.65} \\
                    \rowcolor{lightgray} \cellcolor{white}
                    & \textbf{Ours} & \mathbf{98.38}_{\pm0.11} & \mathbf{99.38}_{\pm0.05} & \mathbf{9.33}_{\pm0.77} & \mathbf{54.46}_{\pm0.44} \\
                \midrule
                \multirow{4}{*}{CIFAR10}        
                    & OE    & 62.23_{\pm0.30} & 57.57_{\pm0.34} & 80.64_{\pm0.98} & 82.67_{\pm0.99} \\
                    & PASCL & 62.33_{\pm0.38} & 57.14_{\pm0.20} & 79.55_{\pm0.84} & 82.30_{\pm1.07} \\
                    & EAT   & 61.62_{\pm0.47} & 55.30_{\pm0.54} & 77.97_{\pm0.77} & 82.61_{\pm0.61} \\
                    \rowcolor{lightgray} \cellcolor{white}
                    & \textbf{Ours} & \mathbf{63.73}_{\pm0.66} & \mathbf{60.72}_{\pm0.74} & \mathbf{85.49}_{\pm1.35} & \mathbf{88.72}_{\pm1.16} \\
                \midrule
                \multirow{4}{*}{\makecell{Tiny \\ ImageNet}} 
                    & OE    & 68.04_{\pm0.37} & 51.66_{\pm0.51} & 76.66_{\pm0.47} & 76.22_{\pm0.61} \\
                    & PASCL & 68.20_{\pm0.37} & 51.53_{\pm0.42} & 76.11_{\pm0.80} & 77.56_{\pm1.15} \\
                    & EAT   & 68.34_{\pm0.28} & 52.79_{\pm0.25} & 74.89_{\pm0.49} & 77.07_{\pm0.39} \\
                    \rowcolor{lightgray} \cellcolor{white}
                    & \textbf{Ours} & \mathbf{71.65}_{\pm0.68} & \mathbf{68.31}_{\pm0.57} & \mathbf{74.94}_{\pm1.07} & \mathbf{83.44}_{\pm1.19} \\
                \midrule
                \multirow{4}{*}{LSUN}            
                    & OE    & 77.10_{\pm0.64} & 61.42_{\pm0.99} & 63.98_{\pm1.38} & 65.64_{\pm1.03} \\
                    & PASCL & 77.19_{\pm0.44} & 61.27_{\pm0.72} & 63.31_{\pm0.87} & 68.05_{\pm1.24} \\
                    & EAT   & 81.09_{\pm0.32} & 67.46_{\pm0.64} & 55.02_{\pm1.20} & 62.07_{\pm0.78} \\
                    \rowcolor{lightgray} \cellcolor{white}
                    & \textbf{Ours} & \mathbf{89.52}_{\pm1.79} & \mathbf{88.38}_{\pm2.15} & \mathbf{37.96}_{\pm4.00} & \mathbf{65.81}_{\pm2.17} \\
                \midrule
                \multirow{4}{*}{Places365}      
                    & OE    & 75.80_{\pm0.45} & 86.68_{\pm0.38} & 65.72_{\pm0.92} & 67.04_{\pm0.49} \\
                    & PASCL & 76.02_{\pm0.21} & 86.52_{\pm0.29} & 64.81_{\pm0.27} & 69.04_{\pm0.90} \\
                    & EAT   & 78.28_{\pm0.31} & 88.20_{\pm0.20} & 60.85_{\pm0.69} & 66.15_{\pm0.68} \\
                    \rowcolor{lightgray} \cellcolor{white}
                    & \textbf{Ours} & \mathbf{67.94}_{\pm1.14} & \mathbf{63.16}_{\pm1.00} & \mathbf{78.79}_{\pm1.22} & \mathbf{85.17}_{\pm0.80} \\
                \midrule
                \multirow{4}{*}{Average}        
                    & OE    & 72.91_{\pm0.68} & 67.16_{\pm0.57} & 68.89_{\pm1.07} & 71.21_{\pm0.84} \\
                    & PASCL & 73.32_{\pm0.32} & 67.18_{\pm0.10} & 67.44_{\pm0.58} & 72.43_{\pm0.66} \\
                    & EAT   & 75.45_{\pm0.83} & 70.87_{\pm0.88} & 64.01_{\pm1.44} & 70.55_{\pm0.98} \\
                    \rowcolor{lightgray} \cellcolor{white}
                    & \textbf{Ours} & \mathbf{77.05}_{\pm0.51} & \mathbf{73.28}_{\pm0.49} & \mathbf{61.99}_{\pm1.04} & \mathbf{77.66}_{\pm0.76} \\
            \bottomrule
            \end{tabular}
        }
        \label{tab:cifar100a}
    \end{subtable}
    
    \vspace{0.5em}
    
    \begin{subtable}{\columnwidth}
    \centering
    \caption{in-distribution classification results in terms of ACC@FPR$n$.}
    \resizebox{\columnwidth}{!}{
        \begin{tabular}{c >{$}c<{$} >{$}c<{$} >{$}c<{$} >{$}c<{$}}
            \toprule
            \multirow{2}{*}{\textbf{Method}} & \multicolumn{4}{c}{\textbf{ACC@FPR$n$ ($\uparrow$)}} \\ 
            & 0 & 0.001 & 0.01 & 0.1 \\ 
            \midrule
            OE & 39.04_{\pm0.37} & 39.07_{\pm0.38} & 39.38_{\pm0.38} & 42.40_{\pm0.44} \\
            PASCL & 43.10_{\pm0.47} & 43.12_{\pm0.47} & 43.39_{\pm0.48} & 46.14_{\pm0.38} \\
            EAT & 46.23_{\pm0.25} & 46.24_{\pm0.25} & 46.38_{\pm0.23} & 48.39_{\pm0.32} \\
            \rowcolor{lightgray}
            \textbf{Ours} & \mathbf{51.03}_{\pm0.47} & \mathbf{51.08}_{\pm0.47} & \mathbf{51.43}_{\pm0.52} & \mathbf{54.69}_{\pm0.52} \\
            \bottomrule
        \end{tabular}
        }
        \label{tab:cifar100b}
    \end{subtable}

    \vspace{0.5em}
    
    \begin{subtable}{\columnwidth}
        \centering
        \caption{Comparison with other methods.}
        \resizebox{\columnwidth}{!}{
            \begin{tabular}{c >{$}c<{$} >{$}c<{$} >{$}c<{$} >{$}c<{$}}
                \toprule
                \textbf{Method} & \textbf{AUROC($\uparrow$)} & \textbf{AUPR($\uparrow$)} & \textbf{FPR95($\downarrow$)} & \textbf{ACC($\uparrow$)} \\
                \midrule
                ST(MSP)      & 61.00 & 57.54 & 82.01 & 40.97 \\
                EnergyOE     & 71.10 & 67.23 & 71.78 & 39.05 \\
                OE           & 72.91_{\pm0.68} & 67.16_{\pm0.57} & 68.89_{\pm1.07} & 39.04_{\pm0.37} \\
                PASCL        & 73.32_{\pm0.32} & 67.18_{\pm0.10} & 67.44_{\pm0.58} & 43.10_{\pm0.47} \\
                EAT          & 74.41_{\pm0.83} & 69.57_{\pm0.88} & 65.05_{\pm1.44} & 46.13_{\pm0.25} \\
                PATT         & \underline{76.25}_{\pm0.90} & \underline{70.37}_{\pm0.90} & \underline{62.94}_{\pm1.40} & \underline{50.07}_{\pm0.30} \\
                DARL         & 70.35_{\pm0.87} & 67.82_{\pm0.78} & 83.73_{\pm1.46} & 39.51_{\pm0.55} \\
                \rowcolor{lightgray}
                \textbf{Ours} & \mathbf{77.05}_{\pm0.51} & \mathbf{73.28}_{\pm0.49} & \mathbf{61.99}_{\pm1.04} & \mathbf{51.03}_{\pm0.51} \\
                \bottomrule
            \end{tabular}
        }
        
        \label{tab:cifar100c} 
    \end{subtable}
    \label{tab:cifar100_abc}
\end{table}

\begin{table*}[t]
    \centering
    \caption{Ablation study on CIFAR10-LT and CIFAR100-LT. DGS, TLA, and EPR are evaluated across multiple metrics.}
    \resizebox{\textwidth}{!}
    {
        \begin{tabular}{c| >{$}c<{$} >{$}c<{$} >{$}c<{$}|>{$}c<{$} >{$}c<{$} >{$}c<{$} >{$}c<{$}|>{$}c<{$} >{$}c<{$} >{$}c<{$} >{$}c<{$}}
            \toprule
            \multirow{2}{*}{$\mathcal{D}_{\text{in}}$} & \multirow{2}{*}{\textbf{DGS}} & \multirow{2}{*}{\textbf{TLA}} & \multirow{2}{*}{\textbf{EPR}} & \multirow{2}{*}{$\mathbf{AUROC(\uparrow)}$} & \multirow{2}{*}{$\mathbf{AUPR(\uparrow)}$} & \multirow{2}{*}{$\mathbf{FPR95(\downarrow)}$} & \multirow{2}{*}{$\mathbf{ACC95(\uparrow)}$} & \multicolumn{4}{c}{$\mathbf{ACC@FPR}n(\uparrow)$} \\
            & & & & & & & & 0 & 0.001 & 0.01 & 0.1 \\
            \midrule
            \multirow{5}{*}{CIFAR10-LT} 
                & \times     & \times     & \times     & 82.35 & 77.84 & 53.00 & 90.24 & 72.05 & 72.11 & 72.56 & 76.21 \\
                & \checkmark & \times     & \times     & 88.37 & 83.06 & 37.89 & 91.48 & 79.18 & 79.23 & 79.66 & 83.46 \\
                & \times     & \checkmark & \times     & 81.96 & 76.00 & 51.21 & 89.56 & 72.49 & 72.56 & 73.00 & 76.93 \\
                & \checkmark & \checkmark & \times     & \underline{90.99} & \underline{88.62} & \underline{33.56} & \underline{93.15} & \underline{84.34} & \underline{84.38} & \underline{84.78} & \underline{88.02} \\
                \rowcolor{lightgray} \cellcolor{white}
                & \checkmark & \checkmark & \checkmark & \mathbf{91.65} & \mathbf{89.81} & \mathbf{31.35} & \mathbf{93.23} & \mathbf{84.80} & \mathbf{84.87} & \mathbf{85.37} & \mathbf{88.51} \\
            \midrule
            \multirow{5}{*}{CIFAR100-LT} 
                & \times     & \times     & \times     & 71.52 & 67.06 & 67.41 & 68.97 & 38.56 & 38.60 & 38.89 & 41.50 \\
                & \checkmark & \times     & \times     & 73.59 & 67.88 & 64.95 & 77.43 & 46.73 & 46.78 & 47.09 & 50.22 \\
                & \times     & \checkmark & \times     & 69.35 & 64.08 & 71.58 & 71.10 & 40.07 & 40.10 & 40.39 & 43.23 \\
                & \checkmark & \checkmark & \times     & \underline{75.84} & \underline{71.70} & \underline{64.58} & \underline{77.25} & \underline{50.70} & \underline{50.76} & \underline{51.07} & \underline{54.13} \\
                \rowcolor{lightgray} \cellcolor{white}
                & \checkmark & \checkmark & \checkmark & \mathbf{77.05} & \mathbf{73.28} & \mathbf{61.99} & \mathbf{77.66} & \mathbf{51.03} & \mathbf{51.08} & \mathbf{51.43} & \mathbf{54.69} \\
            \bottomrule
        \end{tabular}
    }
    \label{tab:ablation_loss}
\end{table*}

\begin{table}[htbp]
    \centering
    \caption{Ablation study of post-hoc OOD method and LTR method on CIFAR100-LT using ResNet18.}
    \newcolumntype{Y}{>{\centering\arraybackslash}X}
    \setlength{\tabcolsep}{3pt} 
    \small 
    \begin{tabularx}{\columnwidth}{cc|YYYY}
        \toprule
        \multirow{2}{*}{\makecell[c]{\textbf{Post-hoc} \\ \textbf{OOD Method}}} & 
        \multirow{2}{*}{\makecell[c]{\textbf{LTR} \\ \textbf{Method}}} & 
        \textbf{AUROC} & \textbf{AUPR} & \textbf{FPR95} & \textbf{ACC} \\
        & & \textbf{($\uparrow$)} & \textbf{($\uparrow$)} & \textbf{($\downarrow$)} & \textbf{($\uparrow$)} \\
        \midrule
        \multirow{4}{*}{Energy} & None & \underline{71.53} & \textbf{66.71} & 67.65 & 38.56 \\
        & Re-weighting & 67.77 & 62.19 & 73.52 & 41.76 \\
        & $\tau$-norm & \textbf{72.24} & \underline{66.17} & \textbf{65.98} & \underline{47.44} \\
        \rowcolor{lightgray} \cellcolor{white}
        & Ours & 70.29 & 64.14 & \underline{67.34} & \textbf{50.71}  \\
        \midrule
        \multirow{4}{*}{ODIN} & None & 71.52 & 67.06 & 67.41 & 38.56 \\
        & Re-weighting & 73.81 & 70.33 & 68.93 & 41.79 \\
        & $\tau$-norm & \underline{75.22} & \underline{71.21} & \underline{63.85} & \underline{47.44} \\
        \rowcolor{lightgray} \cellcolor{white}
        & Ours & \textbf{77.05} & \textbf{73.28} & \textbf{61.99} & \textbf{51.03}  \\
        \bottomrule
    \end{tabularx}
    \label{tab:ablation_ood_ltr}
\end{table}

\subsection{Ablation Study}
\textbf{Analysis of Key Modules.}
To verify the effectiveness of the core components in VMF-GOS, we conduct a comprehensive ablation study on CIFAR10-LT and CIFAR100-LT. Table \ref{tab:ablation_loss} evaluates the individual and joint contributions of DGS, TLA, and EPR. We can infer that
\begin{itemize}
    \item DGS is the primary driver of performance gains. On CIFAR10-LT, DGS elevates the AUROC to 88.37\% and reduces the FPR95 by 15.11\%. Directionally synthesizing virtual outliers provides high-quality negative constraints to effectively tighten ID decision boundaries.
    \item TLA explicitly rectifies long-tailed prior bias. By compensating for class imbalance, TLA ensures high-confidence classification while mitigating classifier favoritism. It significantly improves the discriminative power for tail classes, as evidenced by consistent gains in ACC@FPR$n$ metrics.
    \item EPR enhances the separability of energy distributions: By imposing explicit contrastive constraints, EPR widens the energy gap between ID data and GOS data, effectively suppressing overlap in energy scoring.
\end{itemize}
The joint optimization of all modules achieves optimal robustness: DGS focuses on feature-level boundary tightening, TLA addresses decision-layer bias, and EPR optimizes energy-based metrics. Their synergy enables VMF-GOS to reach superior performance in long-tailed OOD detection.

\begin{figure}[t!]
  \centering 
  \begin{subfigure}[b]{0.499\columnwidth}
    \includegraphics[width=\linewidth]{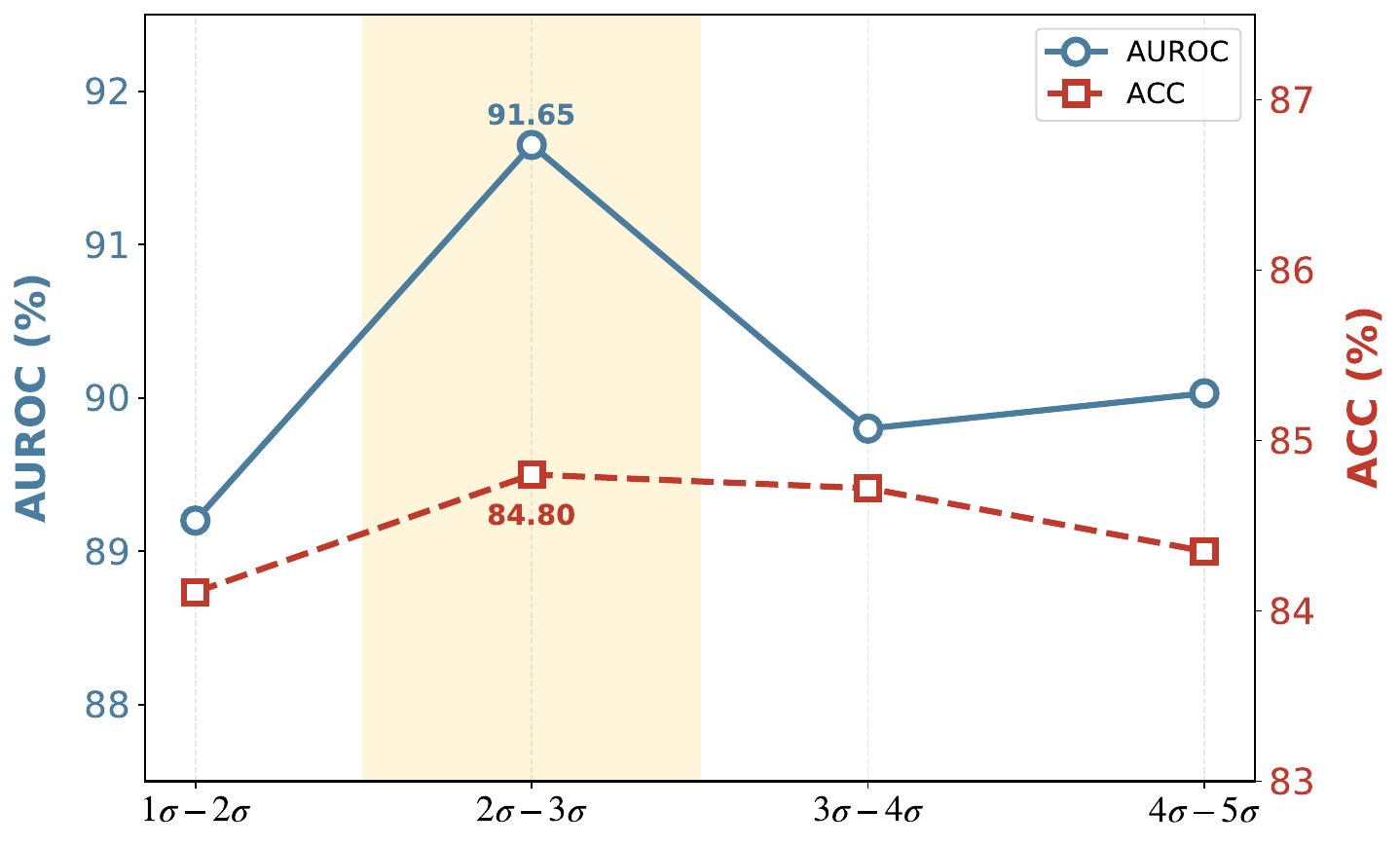}
    \caption{CIFAR-10}
    \label{fig:cifar10_sample}
  \end{subfigure}%
  \hfill 
  \begin{subfigure}[b]{0.499\columnwidth} 
    \includegraphics[width=\linewidth]{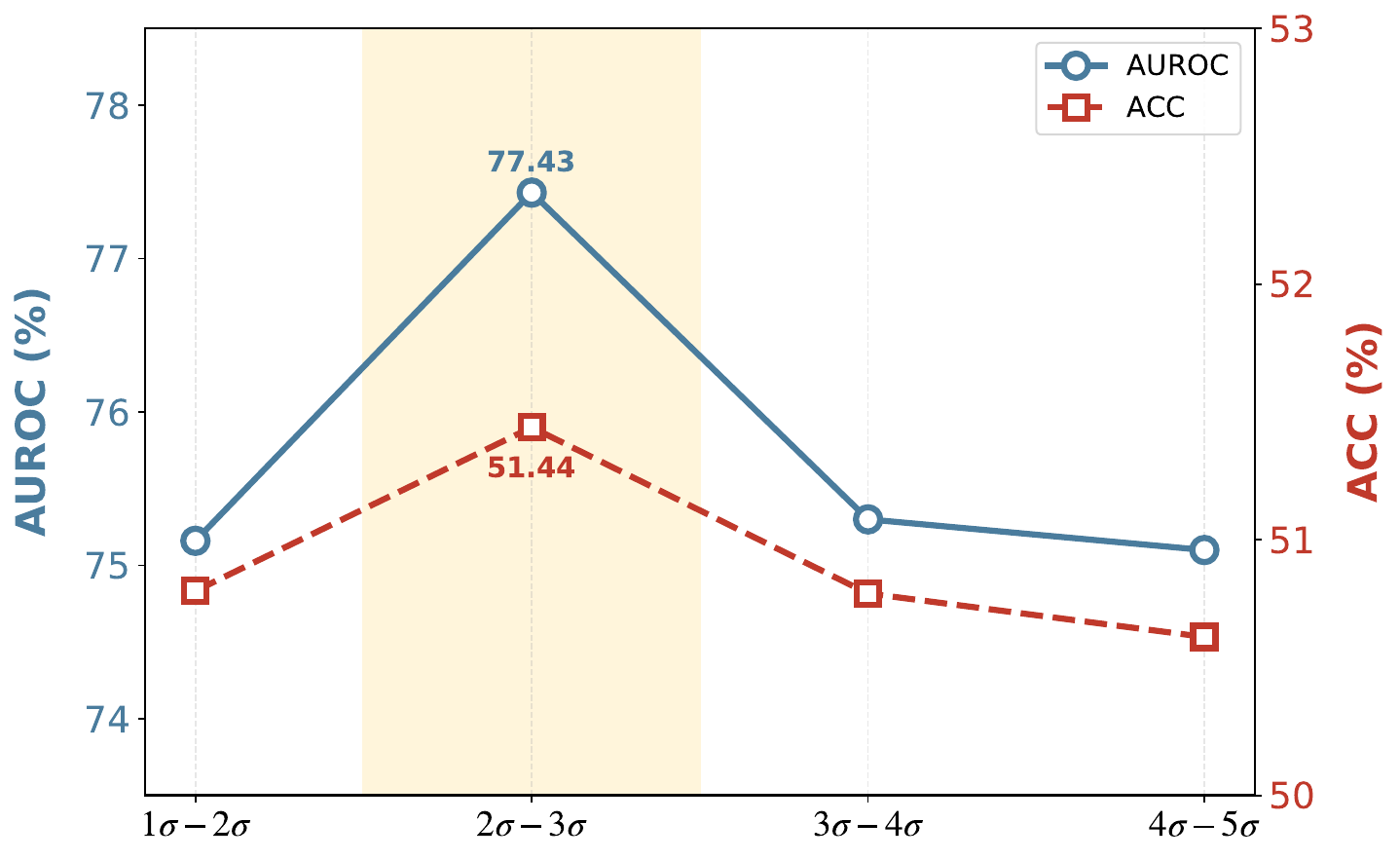}
    \caption{CIFAR-100}
    \label{fig:cifar100_sample}
  \end{subfigure}
  \caption{\textbf{Sensitivity analysis of sampling interval ($\sigma$ range) in the GOS strategy.} Results are reported on CIFAR-10 (left) and CIFAR-100 (right). We evaluate the trends of AUROC (blue solid line) and Accuracy (red dashed line) with respect to different standard deviation ranges of the $\chi^{2}$ distribution. The shaded region indicates our optimal sampling interval of $2\sigma - 3\sigma$.}
  \label{fig:ablation_sample}
\end{figure}
\begin{figure}[t!]
  \centering 
  \begin{subfigure}[b]{0.499\columnwidth}
    \includegraphics[width=\linewidth]{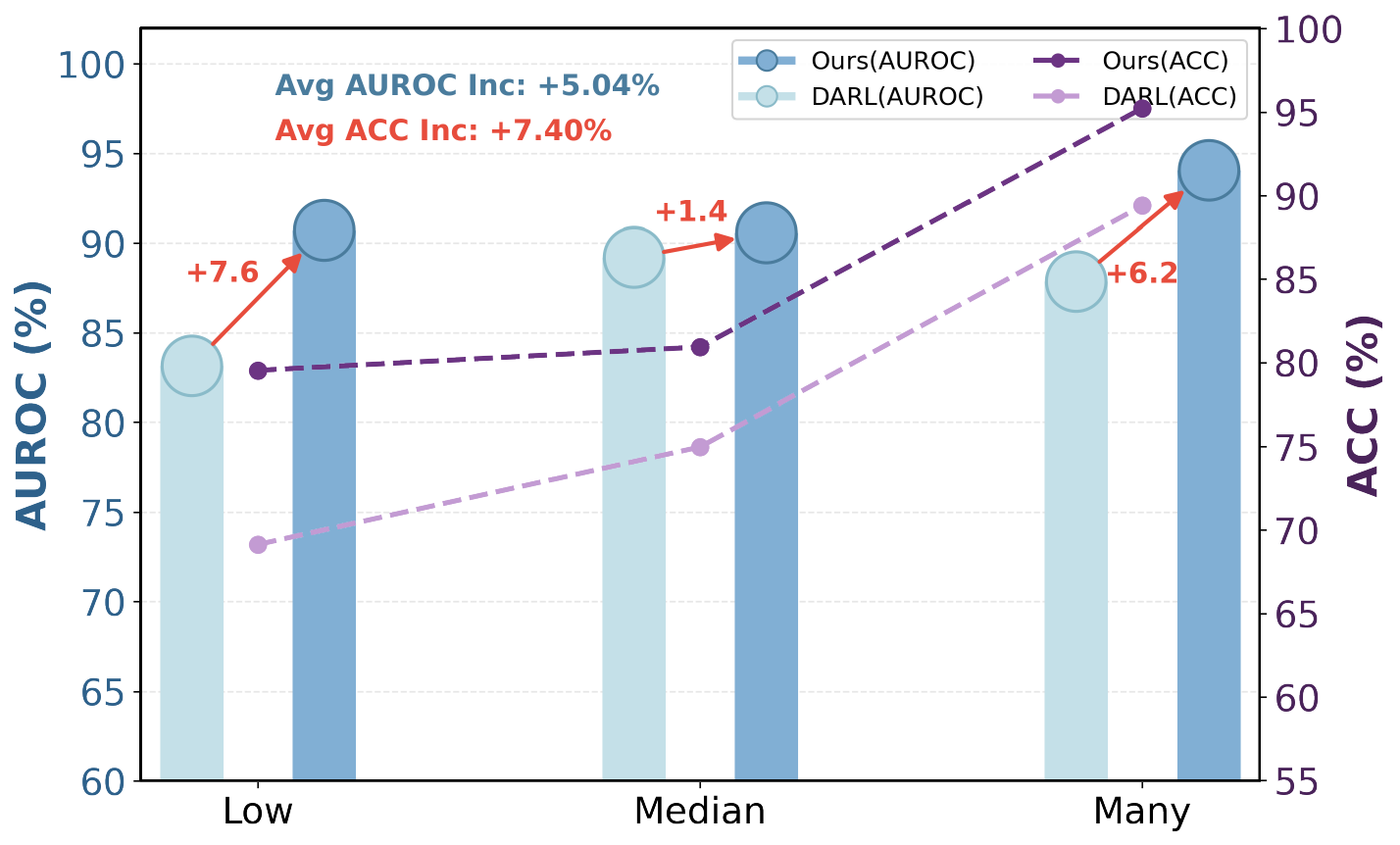}
    \caption{CIFAR-10}
    \label{fig:cifar10_headtail}
  \end{subfigure}%
  \hfill 
  \begin{subfigure}[b]{0.499\columnwidth} 
    \includegraphics[width=\linewidth]{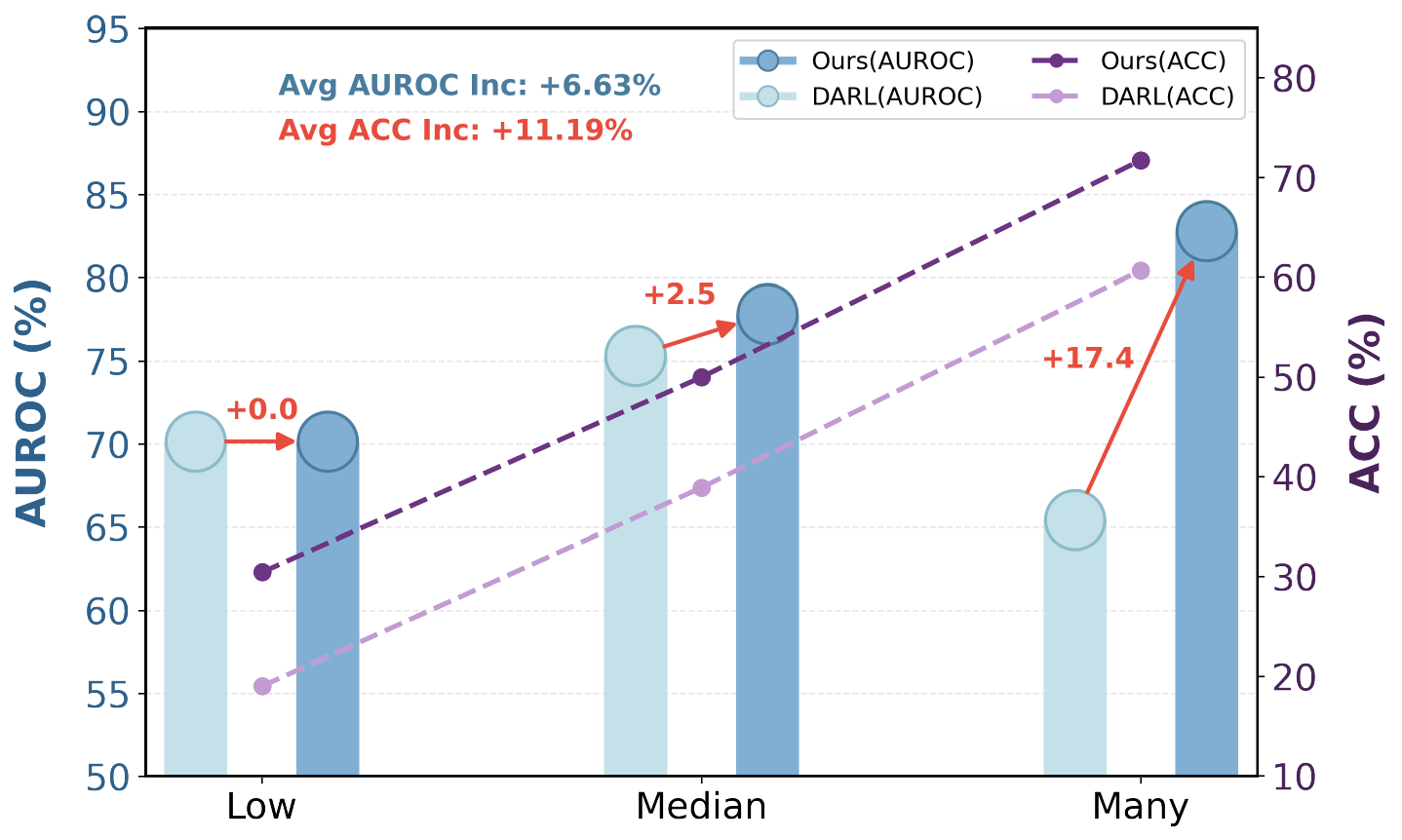}
    \caption{CIFAR-100}
    \label{fig:cifar100_headtail}
  \end{subfigure}
  \caption{\textbf{Performance across class frequencies on long-tailed CIFAR-10 (a) and CIFAR-100 (b).} Classes are categorized into Low, Median, and Many shots based on their sample cardinality. We compare our method against the DARL in terms of AUROC and Accuracy.}
  \label{fig:ablation_headtail}
\end{figure}
\noindent\textbf{Sensitivity Analysis of GOS Sampling Interval.}
We investigate the performance dynamics under different sampling displacements $\xi$, which are parameterized by the standard deviation $\sigma$ of the $\chi^2$ distribution. As illustrated in Figure~\ref{fig:ablation_sample}, the experimental results on both CIFAR-10 and CIFAR-100 datasets strongly validate our selection of $[\mu_{\chi^{2}} + 2\sigma_{\chi^{2}}, \mu_{\chi^{2}} + 3\sigma_{\chi^{2}}]$ as the optimal sampling interval. This consistently superior performance across datasets of varying complexity underscores the robust generalization and the effective boundary refinement capabilities of the GOS strategy.
\newline
\textbf{Performance Breakdown by Class Frequency.}
As illustrated in Figure~\ref{fig:ablation_headtail}, we analyze the performance across Low, Median, and Many-shot categories. It can be seen that our framework consistently outperforms the DARL across the entire frequency spectrum. Notably, we achieve substantial AUROC gains in tail classes while concurrently enhancing the discriminative robustness of head classes (e.g., a $+17.4\%$ improvement in the Many category for CIFAR-100). 
\\
\textbf{Post-hoc OOD Method and LTR Method.}
Table \ref{tab:ablation_ood_ltr} evaluates the compatibility of VMF-GOS with various LTR strategies and post-hoc OOD scoring functions. While our method already achieves competitive results under Energy-based scoring, its combination with ODIN yields a significant performance leap. We attribute this gain to the fact that EPR explicitly shapes a polarized energy landscape, which effectively amplifies the impact of the gradient-based input perturbations utilized by ODIN. Furthermore, compared to specialized LTR methods like $\tau\text{-norm}$, our framework maintains superior ID accuracy while consistently outperforming them in OOD detection robustness.

\section{Conclusion}
In this paper, we propose VMF-GOS, a novel data-free framework addressing the joint challenge of long-tailed recognition and OOD detection. VMF-GOS leverages the intrinsic geometric properties of hyperspherical feature spaces to synthesize high-quality boundary outliers through GOS. Furthermore, we use EPR and TLA to simultaneously widen the discriminative energy margin and rectify class-imbalance biases. Extensive experiments show our method outperforms both existing data-free sota and real-outlier-based approaches.


\bibliography{reference}
\bibliographystyle{icml2026}
\end{document}